\documentclass[abstract=true]{scrartcl}

\usepackage{amsmath}
\usepackage{amsthm}
\usepackage{amssymb}
\usepackage{amsfonts}
\usepackage{hyperref}
\usepackage{color}
\usepackage{dsfont}
\usepackage{enumitem}
\usepackage{algorithmic}
\usepackage{algorithm}
\floatname{algorithm}{\bfseries\sffamily{Algorithm}}

\usepackage{tikz-cd}
\tikzcdset{every label/.append style = {font = \large}}

\usepackage[numbers]{natbib}

\newtheoremstyle{curs} 
  {\topsep}            
  {\topsep}            
  {\itshape}           
  {0pt}                
  {\bfseries\sffamily} 
  {.}                  
  { }                  
  {}                   
  
\newtheoremstyle{ncurs} 
  {\topsep}             
  {\topsep}             
  {}                    
  {0pt}                 
  {\bfseries\sffamily}  
  {.}                   
  { }                   
  {}                    

\theoremstyle{ncurs}
\newtheorem{dfn}{Definition}
\newtheorem{rmk}[dfn]{Remark}
\newtheorem{exa}[dfn]{Example}

\theoremstyle{curs}
\newtheorem{lem}[dfn]{Lemma}
\newtheorem{prop}[dfn]{Proposition}

\newtheorem{thm}[dfn]{Theorem}

\def\R{\mathbb{R}}   
\def\O{\mathcal{O}}  
\def\cX{\mathcal{X}}
\def\cY{\mathcal{Y}}
\def\cZ{\mathcal{Z}}
\def\cH{\mathcal{H}}
\def\cL{\mathcal{L}}
\def\e{\mathrm{e}}   

\newcommand{\norm}[1]{\Vert#1\Vert}

\newcommand{\inner}[1]{\langle#1\rangle}

\newcommand{\id}{\mathrm{id}}

\newcommand{\diag}{\mathrm{diag}}
\newcommand{\tr}{\mathrm{tr}}

\newcommand{\rmD}{\mathrm{D}}
\newcommand{\calL}{\mathcal{L}}

\title{A Derivation of Feedforward Neural Network Gradients Using Fr\'echet Calculus}
\author{Thomas Hamm\\
\normalsize{}Institute for Stochastics and Applications\\
\normalsize{}Faculty 8: Mathematics and Physics\\
\normalsize{}University of Stuttgart\\
\normalsize{}\texttt{\href{mailto:thomas.hamm@mathematik.uni-stuttgart.de}{thomas.hamm@mathematik.uni-stuttgart.de}}}

\begin{document}
\maketitle
\begin{abstract}
We present a derivation of the gradients of feedforward neural networks using Fr\'echet calculus which is arguably more compact than the ones usually presented in the literature. We first derive the gradients for ordinary neural networks working on vectorial data and show how these derived formulas can be used to derive a simple and efficient algorithm for calculating a neural networks gradients. Subsequently we show how our analysis generalizes to more general neural network architectures including, but not limited to, convolutional networks.
\end{abstract}

\section{Preface}

The derivation of the gradients of feedforward neural networks in the literature is usually based on applying the chain rule coordinate-wise at some internal neuron, see for example \cite[Section 6.5]{Goodfellow_DeepLearning} or \cite[Section 11.4]{hastie2009elements}. This approach is already rather cumbersome for the most basic type of neural networks, consisting of a matrix vector product followed by a bias and a non-linearity at each layer, and consequently does not easily generalize to more complex neural networks such as convolutional networks. By using the theory of Fr\'echet calculus we present a derivation of neural network gradients, which is arguably more compact in the case of standard neural networks and also easily generalizes to more complex architectures by exploiting that each layer consists of an affine bilinear transformation
The rest of these notes are organized as follows: In Section \ref{frechet_calc} we introduce the notion of Fr\'echet derivatives including basic properties and an example. In Section \ref{sec:neural_networks} we present our derivation for ordinary neural networks, that is, where each layer transformation is defined by a matrix vector product and a bias followed by a pointwise non-linearity. Subsequently in Section \ref{sec:conv_layers}, we generalize our analysis to arbitrary affine bilinear layer transformations.
\paragraph*{Notation} We denote the set of real numbers by $\R$ and the usual $n$-dimensional Euclidean space by $\R^n$. We denote the set of real $n\times m$-matrices by $\R^{n\times m}$ and for $A\in\R^{n\times m}$ we denote the transposed matrix by $A^\top$. We denote the trace of a square matrix $A$ by $\tr A$ and for some $x\in\R^n$ we denote the diagonal $n\times n$-matrix with entries given by $x$ by $\diag (x)$. Furthermore, we denote the pointwise product, also known as Hadamard product, of two vectors $x,y$ by $x\odot y$. Finally, given two normed spaces $\cX$ and $\cY$ we denote the space of bounded, linear operators from $\cX$ to $\cY$ by $\cL(\cX,\cY)$.
\section{Fr\'echet Calculus}\label{frechet_calc}
\begin{dfn}
Let $\cX, \cY$ be normed spaces  and let $U\subset \cX$ be open. A function $f:U\to \cY$ is called Fr\'echet differentiable at $x\in U$ if there exists a bounded, linear map $\rmD f(x)\in\calL(\cX,\cY)$ such that
\begin{equation*}
\lim_{\norm{h}_\cX\to0} \frac{\norm{f(x+h)-f(x)-\rmD f(x)h}_\cY}{\norm{h}_\cX}=0.
\end{equation*}
Further, $f$ is called Fr\'echet differentiable if $f$ is Fr\'echet differentiable at every $x\in U$.
\end{dfn}

If $U\subset\R^m$ is open and $f:U\to\R^m$ is differentiable then the Jacobian
\begin{equation*}
J_f(x)=\begin{pmatrix}
\frac{\partial f_1(x)}{\partial x_1} & \cdots & \frac{\partial f_1(x)}{\partial x_n} \\
\vdots                               & \ddots & \vdots \\
\frac{\partial f_m(x)}{\partial x_1} & \cdots & \frac{\partial f_m(x)}{\partial x_n} 
\end{pmatrix}, \quad x\in U,
\end{equation*}
is the representing matrix of the linear map $\rmD f(x)\in\calL(\R^n,\R^m)$ with respect to the standard bases in $\R^n$ and $\R^m$.
\paragraph*{Chain Rule.} Let $\cX, \cY, \cZ$ be normed spaces, $U\subset \cX, V\subset \cY$ be open subsets and $f:U\to V$, $g:V\to \cZ$. If $f$ is Fr\'echet differentiable at $x\in U$ and $g$ is Fr\'echet differentiable at $f(x)\in V$ then
\begin{equation}
\rmD (g\circ f)(x)=\rmD g(f(x))\circ \rmD f(x),\label{eqn:chain_rule}
\end{equation}
see \citep[Theorem 2.1]{Coleman_CalculusOnNormed}. Again, if $\cX,\cY,\cZ$ are finite dimensional Euclidean spaces then Equation (\ref{eqn:chain_rule}) translates to the classical chain rule
\begin{equation*}
J_{g\circ f}(x)=J_g(f(x))\cdot J_f(x)
\end{equation*}
in terms of Jacobians.

Fr\'echet derivatives generalize the notion of a derivative to functions on infinite dimensional spaces, as there are in general no coordinates available but are also especially convenient when dealing with functions that have matrix-valued inputs and/or outputs. To demonstrate this usefulness and to get some intuition on the formalism of Fr\'echet calculus we give an example.
\begin{exa}

{ }

\begin{enumerate}
\item Consider the map $f:\R^{n\times k}\times \R^{k\times m}\to \R^{n\times m}$ defined by $f(A,B):=A\cdot B$. For $(H_1,H_2)\in\R^{n\times k}\times \R^{k\times m}$ we have
\begin{align*}
f(A+H_1,B+H_2)&=AB+AH_2+H_1B+H_1H_2 \\
&=f(A,B)+AH_2+H_1B+o(\norm{(H_1,H_2)}),
\end{align*}
as the entries of $H_1\cdot H_2$ only contain terms of second order. Since $(H_1,H_2)\mapsto AH_2+H_1B$ is linear we can immediately conclude that
\begin{equation*}
\rmD f(A,B)(H_1,H_2)=AH_2+H_1B.
\end{equation*}\label{exa:frechet_i}
\item Let $\mathrm{GL}(n)\subset\R^{n\times n}$ be the set of invertible $n\times n$ matrices and define $g:\mathrm{GL}(n)\to \mathrm{GL}(n)$ by $g(A):=A^{-1}$. To find the Fr\'echet derivative of $g$, we define the map $h:\mathrm{GL}(n)\to\R^{n\times n}\times \R^{n\times n}$ by $h(A):=(A,A^{-1})$. Now note that the composition $f\circ h$, where $f$ is defined as above, is constant which implies $\rmD f\circ h=0$ (in the sense of null map). On the other hand, differentiating $f\circ h$ using the chain rule we get
\begin{align*}
0&=\rmD f\circ h(A)(H)=\rmD f(h(A))\circ \rmD h(A)(H)=\rmD f(A,A^{-1})(H,\rmD g(A)(H)) \\
&=A\rmD g(A)(H)+HA^{-1}.
\end{align*}
Solving for $\rmD g(A)$ gives $\rmD g(A)(H)=-A^{-1}HA^{-1}$.
\end{enumerate}
\end{exa}
\begin{rmk}\label{rmk:duality}
If $\cH$ is a (real) Hilbert space, $U\subset\cH$ is open and $f:U\to\R$ is Fr\'echet differentiable at $x\in U$, then by the Riesz representation theorem there exists a unique element $\varphi\in \cH$ satisfying $\langle \varphi, h\rangle_\cH=\rmD f(x)h$ for all $h\in \cH$. In the following, we will often identify $\rmD f(x)$ with $\varphi$. As an example, recall that the Frobenius inner product on the space of matrices $\R^{n\times m}$ is defined by $\langle A,B\rangle =\tr(A^\top B)$ and that this inner product induces the standard Euclidean norm on $\R^{n\times m}$. A Fr\'echet differentiable function $f:\R^{n\times m}\to\R$ also has a derivative and a gradient in the classical sense. By the duality above, if we have determined the Fr\'echet derivative $\rmD f(x)$ and a matrix $G\in\R^{n\times m}$ satisfying $\tr(G^\top H)=\rmD f(x)(H)$ for all $H\in\R^{n\times m}$, we can conclude that $G$ is the gradient of $f$ at $x$ in the classical sense.
\end{rmk}
\paragraph*{Partial Derivatives.} Let $\cX_1,\ldots,\cX_n,\cY$ be normed spaces, let $U_k\subset\cX_k,k=1,\ldots,n$ be open subsets and let $f:\cX_1\times\ldots\times \cX_n\to\cY$. For $k=1,\ldots,n$ we denote the Fr\'echet derivative of $f$ with respect to the $k$-th variable (and the remaining variables held fixed) by $\rmD_k f(x_1,\ldots,x_n)\in\cL(\cX_k,\cY)$. This notation will be important when differentiating expressions of the form $f(g(x_1),x_2)$. For example, if we want to consider the Fr\'echet derivative of this expression with respect to the variable $x_1$, we denote this by $\rmD_{x_1}f(g(x_1),x_2)$ and by the chain rule we have $\rmD_{x_1}\big(f(g(x_1),x_2)\big)=(\rmD_1 f)(g(x_1),x_2)\circ \rmD g(x_1)$, where we have added additional parentheses for clarification.
\section{Neural Networks for Vectorial Data}\label{sec:neural_networks}
Let $W_k\in\R^{d_k\times d_{k-1}} $ and $ b_k\in \R^{d_k}$ for $k=1,\ldots,n$. We refer to the $W_1,\ldots,W_n$ as \emph{weights} and to the $b_1,\ldots,b_n$ as \emph{biases}. We define the layer-wise transformations $f_k:\R^{d_{k-1}}\times \R^{d_k\times d_{k-1}}\times\R^{d_k}\to \R^{d_k}$ by $f_k(x,W_k,b_k):=\sigma_k(W_kx+b_k)$, where $\sigma_k:\R\to\R$ is a non-linearity (or activation function) which is applied coordinate-wise. We further use the abbreviation $Z_k:=(W_1,b_1,\ldots,W_k,b_k)$ for the collection of weights and biases up to layer $k$ and define the activation in the $k$-th layer recursively by $F_0(x)=x$ for $x\in \R^{d_0}$ and 
\begin{equation}
F_k(x,Z_k):=f_k(F_{k-1}(x,Z_{k-1}),W_k,b_k) \quad \text{ for } x\in\R^{d_0} \text{ and } k=1,\ldots,n.\label{eqn:layer_activations}
\end{equation}
The following lemma, which is an immediate consequence of the chain rule, states the derivatives of the layer-wise transformations $f_k$.
\begin{lem}\label{lem:derivates_layers}
For all $k=1,\ldots,n$ we have
\begin{align*}
\rmD_1 f_k(x,W_k,b_k)(h)&=\diag\big(\sigma_k'(W_kx+b_k)\big)W_kh, \quad h\in \R^{d_{k-1}},\\
\rmD_2 f_k(x,W_k,b_k)(H)&=\diag\big(\sigma_k'(W_kx+b_k)\big)Hx, \quad H\in\R^{d_k\times d_{k-1}}, \\
\rmD_3 f_k(x,W_k,b_k)(h)&=\diag\big(\sigma_k'(W_kx+b_k)\big)h, \quad h\in \R^{d_{k}}.
\end{align*}
\end{lem}
\begin{proof}
First define $\Phi(x,W_k,b_k):=W_kx+b_k$. Note that by linearity of the respective operations we have 
\begin{align*}
\rmD_x \Phi(x,W_k,b_k)h&= W_kh, \\
\rmD_{W_k} \Phi(x,W_k,b_k)H&=Hx, \\
\rmD_{b_k} \Phi(x,W_k,b_k)h&=h.
\end{align*}
The assertion now immediately follows by applying the chain rule.
\end{proof}
The expressions in Lemma \ref{lem:derivates_layers} can also be formulated using the Hadamard product instead of multiplication with diagonal matrices. The disadvantage of the Hadamard product however is, that it is not associative when combined with ordinary matrix multiplication. To simplify our subsequent analysis by not having to consider the order of operation, we prefer the usage of diagonal matrices. In the following proposition we combine Lemma \ref{lem:derivates_layers} with the chain rule to derive the gradients of the neural network with respect to an arbitrary weight and bias.
\begin{prop}\label{prop:gradients_preloss}
Let $D_k:=\diag\big(\sigma_k'(W_kF_{k-1}(x,Z_{k-1})+b_k)\big) $ for $ k=1,\ldots,n$. Then for all $k=1,\ldots,n$ we have
\begin{align*}
\rmD_{W_k}F_n(x,Z_n)(H)&=D_nW_n\cdot\ldots\cdot D_{k+1}W_{k+1} D_kHF_{k-1}(x,Z_{k-1}), \\
\rmD_{b_k}F_n(x,Z_n)(h)&=D_nW_n\cdot\ldots\cdot D_{k+1}W_{k+1} D_kh.
\end{align*}
\end{prop}
\begin{proof}
For better readability we will suppress the arguments of the function $F_k$ occasionally. Note that 
\begin{equation*}
\rmD_{W_k}F_m(x,Z_m)=\rmD_{W_k} f_m(F_{m-1},W_m,b_m)=\rmD_1 f_m(F_{m-1},W_m,b_m)\circ \rmD_{W_k} F_{m-1}(x,Z_{m-1})
\end{equation*}
for $m>k$ and $\rmD_{W_k}F_k(x,Z_k)=\rmD_2 f_k(F_{k-1},W_k,b_k)$. Now applying the chain rule gives
\begin{align*}
&\rmD_{W_k}F_n(x,Z_n) \\
=\,&\rmD_1 f_n(F_{n-1},W_n,b_n)\circ\ldots\circ \rmD_1 f_{k+1}(F_{k},W_{k+1},b_{k+1})\circ \rmD_2 f_k(F_{k-1},W_k,b_k).
\end{align*}
Combining this with Lemma \ref{lem:derivates_layers} yields the assertion. $\rmD_{b_k} F_n(x,Z_n)$ can be derived analogously.
\end{proof}
Let $\ell:\R^{d_n}\times\R^{d_n}\to\R$ be a loss function, that is the discrepancy between our prediction $F_n(x,Z_n)$ and the true label $y\in\R^{d_n}$ is measured by $\ell(y,F_n(x,Z_n))$. Let
\begin{equation*}
L(y,t):=\left( \frac{\partial \ell (y,t)}{\partial t_1},\ldots,\frac{\partial \ell (y,t)}{\partial t_{d_n}} \right) \in\R^{1\times d_n}
\end{equation*}
be the Jacobian of the map $\R^{d_n}\ni t\mapsto \ell(y,t)\in\R$ for fixed $y\in\R^{d_n}$. In other words, we have $\rmD_2\ell(y,t)h=L(y,t)\cdot h$ for all $h\in\R^{d_n}$.
\begin{exa}
If $\ell$ is the least squares loss defined by $\ell(y,t):=\norm{y-t}^2$ for $y,t\in\R^{d_n}$, then we have $L(y,t)=2(t-y)^\top$.
\end{exa}
Finally, we can combine our analysis so far to derive the gradients of a loss of a neural network.
\begin{thm}\label{thm:nn_grads}
For all $k=1,\ldots,n$ we have
\begin{align*}
\rmD_{W_k}\ell(y,F_n(x,Z_n))^\top&=F_{k-1}(x,Z_{k-1})L(y,F_n(x,Z_n))D_nW_n\cdot\ldots\cdot D_{k+1}W_{k+1} D_k, \\
\rmD_{b_k} \ell(y,F_n(x,Z_n))^\top &= L(y,F_n(x,Z_n))D_nW_n\cdot\ldots\cdot D_{k+1}W_{k+1} D_k. 
\end{align*}
\end{thm}
For clarification of notation note that for $k=n$ the product $D_nW_n\cdot\ldots\cdot D_{k+1}W_{k+1}$ is empty and we therefore have $\rmD_{W_n} \ell(y,F_n))^\top= F_{n-1}(x,Z_{n-1})L(y,F_n)D_n$ as well as $\rmD_{b_n}\ell(y,F_n)^\top=L(y,F_n)D_n$.
\begin{proof}
Again, we will suppress the arguments of the function $F_n$ occasionally for better readability. By Proposition \ref{prop:gradients_preloss} we have for all $H\in \R^{d_k\times d_{k-1}}$ and $k=1,\ldots,n$ that
\begin{align*}
&\rmD_{W_k}\ell(y,F_n(x,Z_n))(H)=\rmD_2\ell(y,F_n)\circ \rmD_{W_k}F(x,Z_n)(H) \\
=\,&L(y,F_n)D_nW_n\ldots D_{k+1}W_{k+1}D_kHF_{k-1} \\
=\,&\tr\big( F_{k-1}L(y,F_n)D_nW_n\ldots D_{k+1}W_{k+1}D_kH \big)
\end{align*}
where in the last step we have used that the trace is invariant under cyclic permutations. The statement now follows by the duality mentioned in Remark \ref{rmk:duality}. The derivation of $\rmD_{b_k}F(x,Z_n)$ is straightforward.
\end{proof}
\begin{algorithm}[h!]
\caption{Computation of gradients $G_1,g_1,\ldots,G_n,g_n$ of the loss of a neural network $F_n$, defined by (\ref{eqn:layer_activations}), with respect to $W_1,b_1,\ldots,W_n,b_n$, respectively, for an input $x\in\R^{d_0}$ and a target label $y\in\R^{d_n}$. In the forward pass set $\sigma_0:=\id$.}
\begin{algorithmic}
\STATE $a_0\gets x$  \hfill $\diamond $ forward pass
\FOR{$k=1,\ldots,n$} 
\STATE{$a_k\gets W_k\cdot\sigma_{k-1}(a_{k-1})+b_k$} 
\ENDFOR
\STATE $g_n\gets \diag(\sigma_n'(a_n))\cdot L(y,\sigma_n(a_n))^\top$ \hfill $\diamond $ backward pass\\
$G_n\gets g_n \cdot\sigma_{n-1}(a_{n-1})^\top$ 
\FOR{$k=n-1,\ldots,1$}
\STATE $g_k\gets \diag(\sigma_k'(a_k))\cdot W_{k+1}^\top\cdot g_{k+1}$ \\
$G_k\gets g_k\cdot \sigma_{k-1}(a_{k-1})^\top$
\ENDFOR
\end{algorithmic}
\label{alg:NNgrads}
\end{algorithm}
The recursive pattern of the expressions for the gradients in Theorem \ref{thm:nn_grads} easily yield an efficient procedure for computing them, see Algorithm \ref{alg:NNgrads}.

\newpage
\begin{rmk}
There are several ways to improve the efficiency of Algorithm \ref{alg:NNgrads}:
\begin{enumerate}
\item To improve memory efficiency of Algorithm \ref{alg:NNgrads}, the application of $\sigma_k$ on $a_k$ is recomputed during the backward pass. As the application of the non-linearity is very inexpensive, this is often preferable to storing both $a_k$ and $\sigma(a_k)$ during the forward pass. If memory consumption is no limiting constraint one can store both the values $\sigma_k(a_k)$ and $\sigma_k'(a_k)$ instead of $a_k$ during the forward pass to maximize the speed of computation.

If $\sigma_k=\max\{0,\cdot\}$ is the ReLU activation function this tradeoff between memory and time complexity can be avoided by utilizing the identity $\sigma_k'(\sigma_k(x))=\sigma_k'(x)$ for all $x\in\R$, where $\sigma_k'=\mathds{1}_{(0,\infty)}$. During the forward pass the values $F_k:=\sigma(a_k)$ can be stored instead. Accordingly, the variables $\sigma_k'(F_k)$ are used instead of $\sigma_k'(a_k)$ during the backward pass. The same holds for the sigmoid activation function $\sigma_k(x)=1/(1+\e^x)$, which satisfies $\sigma_k'(x)=\sigma_k(x)(1-\sigma_k(x))$. Here we can replace $\sigma_k'(a_k)$ with $F_k\odot(\mathbf{1}-F_k)$ in the backward pass, where $\odot$ denotes the pointwise product of vectors and $\mathbf{1}$ denotes the vector of ones of appropriate size. Also, the hyperbolic tangent activation function $\sigma_k=\tanh$ satisfies $\sigma_k'(x)=1-\sigma_k(x)^2$, so we can replace $\sigma_k'(a_k)$ with $\mathbf{1}-F_k\odot F_k$ during the backward pass.
\item As the gradients for the weights $W_k$ are the outer product of the $d_{k}$-dimensional vector $g_k$ and the $d_{k-1}$-dimensional vector $\sigma_{k-1}(a_{k-1})$ they can be stored with $\O(d_{k-1}+d_k)$ space instead of $\O(d_{k-1}d_k)$ space.
\item At the end of each loop cycle in the backward pass the variable $a_k$ can be deleted. If the weights and biases are updated immediately via $W_k\gets W_k-\eta G_k$ and $b_k\gets b_k-\eta g_k$, where $\eta>0$ is the learning rate, the weight gradient $G_k$ can be deleted immediately.
\end{enumerate}
\end{rmk}

\section{General Layer Transformations}\label{sec:conv_layers}
Note that the only properties that we have actually used in the analysis of Section \ref{sec:neural_networks} was that the transformation in each layer consists of a bilinear map (with respect to the weights and the layers input) and a bias, followed by a pointwise non-linearity. In this section we want to generalize the results of Section \ref{sec:neural_networks} to general bilinear transformations per layer, which most importantly include convolutional layers. To emphasize the similarity to the analysis in Section \ref{sec:neural_networks} we try to keep the notation in this section as close as possible to the notation in Section \ref{sec:neural_networks}. In the following, given an index set $I$ let $\R^I$ be the real vector space consisting of vectors $x=(x_i)_{i\in I}$ with entries indexed by $I$. Equipped with the inner product
\begin{equation*}
\inner{x,y}_{\R^{I}}:=\sum_{i\in I}x_iy_i \quad \text{ for } x,y\in\R^I,
\end{equation*}
$\R^I$ becomes a Hilbert space.

 In convolutional networks one usually has layer data of the form $\R^{d_1\times d_2\times c}$ for $d_1\times d_2$-pixel images with $c$ channels. We include these types of layers using the index set $I=\{1,\ldots,d_1\}\times\{1,\ldots,d_2\}\times\{1,\ldots,c\}$

We proceed by describing the neural networks layer transformations. Let $I_0,I_k,J_k,K_k, $ for $k=1,\ldots,n$ be index sets. Let $C_k:\R^{I_{k-1}}\times \R^{J_k}\to\R^{I_k}$ be a bilinear map and $\iota_k:\R^{K_k}\to\R^{I_k}$ a linear map. We again define the layer-wise transformation by
\begin{equation*}
f_k(x,W_k,b_k):=\sigma_k\big(C_k(x,W_k)+\iota_k(b_k)\big)
\end{equation*}
and $F_k$ as in Equation (\ref{eqn:layer_activations}) using the adjusted $f_k$.

To explain the necessity of $\iota$, consider again the example of a convolutional layer of the form $\R^{d_1\times d_2\times c}$. Here, one usually has only one bias per channel. We can realize this as follows: For $k=1,\ldots,c$ let $\beta^{(k)}\in\R^{d_1\times d_2\times c}$ be defined by $\beta_{i,j,k}^{(k)}=1$ for all $i,j$ and $\beta_{i,j,l}^{(k)}=0$ for all $i,j$ if $l\neq k$. Now define $\iota:\R^c\to\R^{d_1\times d_2\times c}$ by
\begin{equation*}
\iota(b):=\sum_{k=1}^c b_k\beta^{(k)}.
\end{equation*}

Similar to Lemma \ref{lem:derivates_layers} one can easily show using the chain rule that
\begin{align*}
\rmD_1 f_k(x,W_k,b_k)(h)&=\sigma_k'\big( C_k(x,W_k)+\iota_k(b_k)\big) \odot C_k(h,W_k), \\
\rmD_2 f_k(x,W_k,b_k)(H)&=\sigma_k'\big( C_k(x,W_k)+\iota_k(b_k)\big) \odot C_k(x,H), \\
\rmD_3 f_k(x,W_k,b_k)(h)&=\sigma_k'\big( C_k(x,W_k)+\iota_k(b_k)\big) \odot \iota_k(h).
\end{align*}
The following proposition is an analogue of Proposition \ref{prop:gradients_preloss}.
\begin{prop}\label{prop:derivatives_conv_net}
For $k=1,\ldots,n$ let $T_k:\R^{I_{k-1}}\to\R^{I_k}$ and $S_k:\R^{J_k}\to\R^{I_k}$ be defined by
\begin{align*}
T_k h:=D_k\odot C_k(h,W_k),  \quad S_kH:= D_k\odot C_k(F_{k-1},H), \quad R_kh:= D_k\odot \iota_k(h),
\end{align*}
where $D_k:=\sigma_k'\big( C_k(F_{k-1}(x,Z_{k-1}),W_k)+\iota_k(b_k)\big)$. Then for $k=1,\ldots,n$ we have
\begin{align*}
\rmD_{W_k} F_n(x,Z_n)&=T_n\circ\ldots\circ T_{k+1}\circ S_k, \\
\rmD_{b_k} F_n(x,Z_n)&=T_n\circ\ldots\circ T_{k+1}\circ R_k.
\end{align*}
\end{prop}

Let $\cH_1$ and $\cH_2$ be Hilbert spaces and let $T\in\cL(\cH_1,\cH_2)$ be a bounded, linear operator. The adjoint operator $T^*\in\cL(\cH_2,\cH_1)$ of $T$ is the unique operator satisfying $\inner{Tx,y}_{\cH_2}=\inner{x,T^*y}_{\cH_1}$ for all $x\in\cH_1$ and $y\in\cH_2$.

Let $\ell:\R^{I_n}\times \R^{I_n}\to\R$ be a loss function and
\begin{equation*}
L(y,t):=\left( \frac{\partial \ell(y,t)}{\partial t_i}\right)_{i\in I_n}\in\R^{I_n}.
\end{equation*}

\begin{thm}\label{thm:conv_net_grads}
Let $T_k,S_k,R_k,k=1,\ldots,n$ be defined as in Proposition \ref{prop:derivatives_conv_net}. Then for all $k=1,\ldots,n$ we have
\begin{align*}
\rmD_{W_k}\ell(y,F_n(x,Z_n))=S_k^*\circ T_{k+1}^*\circ\ldots\circ T_n^*(L(y,F_n)), \\
\rmD_{b_k}\ell(y,F_n(x,Z_n))=R_k^*\circ T_{k+1}^*\circ\ldots\circ T_n^*(L(y,F_n)).
\end{align*}
\end{thm}
\begin{proof}
By the chain rule we have
\begin{align*}
\rmD_{W_k}\ell(y,F_n(x,Z_n))H&=\rmD_2\ell(y,F_n)\circ \rmD_{W_k}\ell(y,F_n)(H) \\
&=\inner{L(y,F_n),\rmD_{W_k}\ell(y,F_n)(H)}_{\R^{I_n}}.
\end{align*}
The result now follows from Proposition \ref{prop:derivatives_conv_net} combined with Remark \ref{rmk:duality}.
\end{proof}
For operators $T:\R^I\to\R^J$ the adjoint $T^*:\R^J\to\R^I$ can be computed explicitly. To this end, we denote the standard basis of $\R^I$ by $e_i,i\in I$, where each $e_i$ is defined by having a 1 at index $i$ and 0 else. For $x\in \R^J$ we have
\begin{equation}
T^*x=\sum_{i\in I} \inner{T^*x,e_i}_{\R^I} e_i=\sum_{i\in I} \inner{x,Te_i}_{\R^J} e_i.\label{eqn:adjoint_action}
\end{equation}

Let $C_k^\dagger:\R^{I_k}\times \R^{J_k}\to\R^{I_{k-1}}$ be the adjoint of $C_k$ with respect to its first argument and $C_k^{\dagger\dagger}:\R^{I_{k-1}}\times \R^{I_k}\to\R^{J_k}$ with respect to its second argument. For the operators $T_k,S_k,R_k$ from Proposition \ref{prop:derivatives_conv_net} and Theorem \ref{thm:conv_net_grads} we then have
\begin{align*}
T_k^*h&=C_k^{\dagger}\big( D_k\odot h,W_k \big), \\
S_k^*H&=C_k^{\dagger\dagger}\big( F_{k-1},D_k\odot H \big), \\
R_k^*h&=\iota_k^*\big(D_k\odot h\big),
\end{align*}
which can be computed using Equation (\ref{eqn:adjoint_action}). Using this notation, Theorem \ref{thm:conv_net_grads} then easily yields Algorithm \ref{alg:NNgrads_conv}.
\begin{algorithm}[h!]
\caption{Computation of gradients $G_1,g_1,\ldots,G_n,g_n$ of the loss of a neural network $F_n$ as defined in Section \ref{sec:conv_layers}, with respect to $W_1,b_1,\ldots,W_n,b_n$, respectively, for an input $x\in\R^{I_0}$ and a target label $y\in\R^{I_n}$. In the forward pass set $\sigma_0:=\id$.}
\begin{algorithmic}
\STATE $a_0\gets x$  \hfill $\diamond $ forward pass
\FOR{$k=1,\ldots,n$} 
\STATE{$a_k\gets C_k(\sigma_{k-1}(a_{k-1}),W_k)+\iota_k(b_k)$} 
\ENDFOR
\STATE $T\gets \sigma_n'(a_n)\odot L(y,\sigma_n(a_n))$ \hfill $\diamond $ backward pass\\
$g_n\gets \iota_n^*(T)$ \\
$G_n\gets C_n^{\dagger\dagger}(\sigma_{n-1}(a_{n-1}),T)$ 
\FOR{$k=n-1,\ldots,1$}
\STATE $T\gets C_{k+1}^\dagger(T,W_{k+1})\odot \sigma_k'(a_k)$ \\
$g_k\gets \iota_k^*(T)$ \\
$G_k\gets C_k^{\dagger\dagger}(\sigma_{k-1}(a_{k-1}),T)$
\ENDFOR
\end{algorithmic}
\label{alg:NNgrads_conv}
\end{algorithm}

\bibliography{bib}

\end{document}